\documentclass{article}

\usepackage{microtype}
\usepackage{graphicx}
\usepackage{subfigure}
\usepackage{booktabs} %

\usepackage{hyperref}
\usepackage{adjustbox}
 \usepackage[frozencache,cachedir=minted-cache]{minted}
\RecustomVerbatimEnvironment{Verbatim}{BVerbatim}{}

 \usepackage[accepted]{icml2023}

\usepackage{amsmath}
\usepackage{amssymb}
\usepackage{mathtools}
\usepackage{amsthm}
\usepackage[shortlabels]{enumitem}

\usepackage[capitalize,noabbrev]{cleveref}
\Crefname{figure}{Fig.}{Figs.}
\Crefname{equation}{Eq.}{Eqs.}
\Crefname{section}{Sec.}{Secs.}
\usepackage{amsfonts}

\usepackage{booktabs}

\newcommand{\reals}{\mathbb{R}}
\def\defeq{\triangleq}
\newcommand{\inner}[2]{\langle{#1},{#2}\rangle}

\theoremstyle{plain}
\newtheorem{theorem}{Theorem}[section]

\theoremstyle{definition}
\newtheorem{definition}[theorem]{Definition}

\theoremstyle{remark}

\usepackage[textsize=tiny]{todonotes}

\newcommand{\ebf}[1]{\textcolor{cyan}{[EBF: #1]}}
\newcommand{\ourmodel}{\textsc{MultiresNet}}
\newcommand{\treeselect}{\textsc{TreeSelect}}
\newcommand{\ourlayer}{\textsc{MultiresLayer}}
\newcommand{\ourblock}{\textsc{MultiresBlock}}
\newcommand{\treeblock}{\textsc{MultiresBlock}}
\newcommand{\ourconv}{\textsc{MultiresConv}}
\newcommand{\prepar}{\vspace{-0.1in}}

\icmltitlerunning{~ \hfill Sequence Modeling with Multiresolution Convolutional
Memory \hfill \thepage}

\begin{document}

\twocolumn[
\icmltitle{Sequence Modeling with Multiresolution Convolutional Memory}

\icmlsetsymbol{equal}{*}

\begin{icmlauthorlist}
\icmlauthor{Jiaxin Shi}{stanford}
\icmlauthor{Ke Alexander Wang}{stanford}
\icmlauthor{Emily B. Fox}{stanford,czhubsf}
\end{icmlauthorlist}

\icmlaffiliation{stanford}{Stanford University}
\icmlaffiliation{czhubsf}{CZ Biohub SF}

\icmlcorrespondingauthor{Jiaxin Shi}{jiaxins@stanford.edu}

\icmlkeywords{Machine Learning, ICML}

\vskip 0.3in
]

\printAffiliationsAndNotice{}  %

\begin{abstract}
\comment{
\ebf{needs updating: Modeling the dependence of sequential data over a long time horizon requires an effective memory of the past.  
Recurrent neural networks 
have a built-in memory in the form of hidden states. 
Convolutions, despite also widely used for sequence modeling, typically lack such a structure. 
In this work, we introduce a class of neural network architectures with a hidden memory constructed by multiresolution convolutions. 
We provide a theoretical framework based on wavelets for deriving such neural networks. 
Our final architecture has a hierarchical dilated convolution structure similar to that of WaveNets but also has a few key differences, e.g., filters are tied for all but the last layers, and the output is generated through a tree-based convolution. 
We apply our multiresolution convolution layers to a variety of sequence modeling tasks...
}}

Efficiently capturing the long-range patterns in sequential data sources salient to a given task---such as classification and generative modeling---poses a fundamental challenge. Popular approaches in the space tradeoff between the memory burden of 
brute-force enumeration and comparison,
as in transformers, the computational burden of complicated sequential dependencies, %
as in recurrent neural networks, or the parameter burden of convolutional networks with many or large filters.  
We instead take inspiration from wavelet-based multiresolution analysis to define a new building block for sequence modeling, which we call a \ourlayer. 
The key component of our model is the \emph{multiresolution convolution}, capturing multiscale trends in the input sequence. 
Our \ourconv\ 
can be implemented with \emph{shared} filters across a dilated causal convolution tree. 
Thus it %
garners the computational advantages of convolutional networks and the principled theoretical motivation of wavelet decompositions.
Our \ourlayer\ is straightforward to implement, requires significantly fewer parameters, and maintains at most a $\mathcal{O}(N\log N)$ memory footprint for a length $N$ sequence.
Yet, by stacking such layers, our model yields state-of-the-art performance on a number of sequence classification and autoregressive density estimation tasks using CIFAR-10, ListOps, and PTB-XL datasets.

\end{abstract}

\section{Introduction}

A key challenge in sequence modeling is summarizing, or \emph{memorizing}, long-term patterns in data informative for a particular task, such as classification, forecasting, or clustering. %
By definition, patterns are higher-level structures in the data that arise from multiple timesteps.
However, patterns can occur at multiple levels, corresponding to different timescales.  For example, in studying energy consumption, patterned variations may occur within a day, between days, and quarterly. Similar salient multiscale trends appear in physiological time series such as dysfunctional glucose patterns in diabetic patients and anomalous heart beats in arrhythmic patients.  Audio signals of speech may be described in terms of utterances, phonemes, and subphonemes.  And, the multiscale structure of images and video has been well-studied.  Even for data sources without an explicit multiscale interpretation, multiscale modeling approaches can provide an efficient mechanism for capturing long-range patterns.  %

In this paper, we propose a general and reusable building block for sequence modeling---%
\ourlayer---leveraging a multiscale approach to memorize past data.  
We view memory through the lens of multiresolution analysis (MRA)~\cite{willsky2002multiresolution}, with a particular emphasis on wavelet analysis, a powerful tool from signal processing for compression, denoising, feature extraction, and more~\cite{jawerth1994overview,akansu2001multiresolution}. As discussed in \cref{sec:wavelet-background}, wavelet analyses can be computed in a computationally efficient manner and interpreted as a series %
of convolutions. 
However, our use of wavelets is a design choice and other MRA techniques could likewise be considered for 
memorizing patterns at different timescales.

Taking inspiration from wavelets, the 
key component of \ourlayer\ is a \emph{multiresolution convolution} operation (\ourconv) that retains the overall tree-structure of MRA. 
We show that constructing a memory of the past at each timestep of the sequence using \ourconv\ can be collectively implemented as a stack of carefully-placed dilated causal convolutions with filters shared between levels. 
In contrast to traditional wavelet analysis, however, we \emph{learn} the filters and do so end-to-end. When we fix the filters to pre-defined wavelet filters, the \ourconv\ reduces to a traditional discrete wavelet transform, though we show the benefits of learning the filters in \cref{sec:exp-ablation}.  
The basic \ourconv\ building block can be stacked in a multitude of ways that are common in deep learning models (e.g., across multiple channels, vertically as multiple layers, etc.). 
Our model resembles WaveNet~\cite{oord2016wavenet} in the use of tree-structured dilated convolutions. 
However, our principle-guided design has distinct skip-connection structures and filter sharing patterns, resulting in significantly better parameter efficiency and performance (see \cref{sec:related-work} for further details).

There is a rapidly growing literature on machine learning for sequence modeling. %
Popular classes of approaches include variants of recurrent networks~\citep{hochreiter1997long}, self-attention networks~\citep{vaswani2017attention}, and state-space models~\citep{gu2021efficiently}. See \cref{sec:related-work} for further discussion.
Our \ourlayer\ has key advantages over this body of past work:
\begin{itemize}%
\vspace{-0.1in}
\itemsep0em
    \item \textbf{Architecture simplicity:} The workhorse of our layer is simple dilated convolutions and linear transforms.
    \item \textbf{Efficient training:} Our layer parallelizes easily across hardware accelerators implementing convolutions.
    \item \textbf{Parameter efficiency:} Our layer reuses filters across the stack of depthwise dilated convolutions.
\end{itemize}
\vspace{-0.1in}
Likewise, by leveraging an MRA structure, we start from a principled and interpretable framework for thinking about memory in sequence modeling. Furthermore, we can lean on the vast MRA literature for modeling generalizations, such as shift-invariant wavelet transforms~\cite{kingsbury1998dual,selesnick2005dual} for shift-invariant representation learning, scaling to multiple input dimensions, etc.

Our empirical evaluation covers 
sequential image classification and autoregressive generative modeling (CIFAR-10), reasoning on syntax trees (ListOps), and multi-label classification of electrocardiogram (PTB-XL). 
We also note that our proposed \ourconv s can readily be applied and extended to other tasks such as representation learning 
and long-term forecasting. 
Likewise, although we focus on sequence analysis, the ideas we propose generalize to other data domains with multiresolution structure, such as images and videos.
Exploring the application of \ourlayer\ in these settings is an exciting future direction.

\section{Background: Wavelet Decompositions}
\label{sec:wavelet-background}

\begin{figure}[t]
\includegraphics[width=0.9\linewidth]{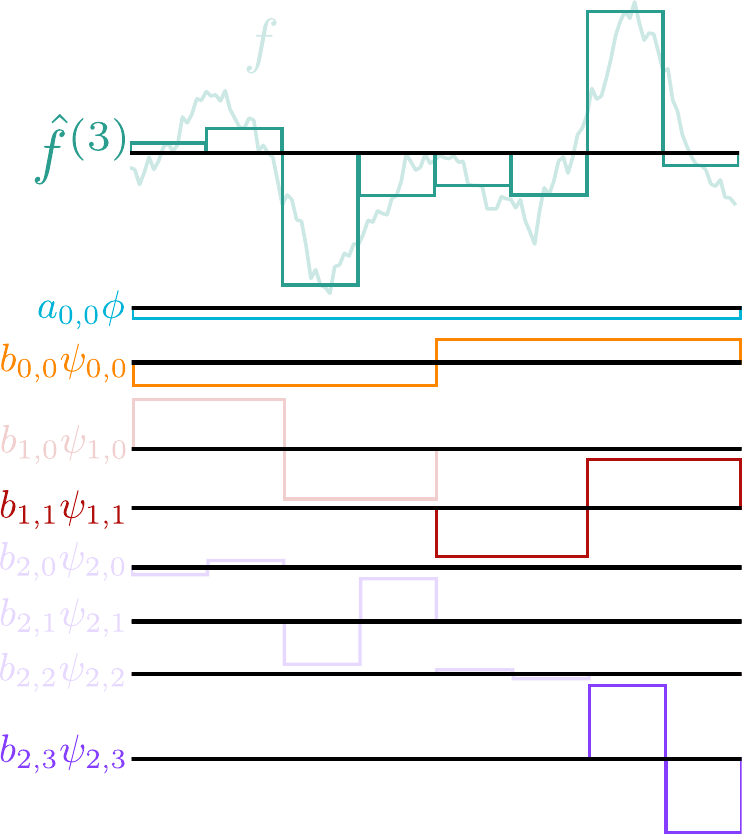}
\centering
\vspace{-2mm}
\caption{In standard MRA, we approximate the continuous signal $f$ with $\hat{f}^{(j)}$. Here, we visualize $\hat{f}^{(3)}$ and its decomposition into a sum of functions that capture structures of increasing resolution over a binary-tree-partitioned input space, corresponding to \cref{eq:wavelet-series}. Components belonging to the same level have the same color. The non-transparent components of each resolution level retain the most recent information in the decomposition. Retaining only these components corresponds to our \ourlayer\ with a ``resolution fading'' \textsc{TreeSelect}.
} %
\vspace{-4mm}
\label{fig:components}
\end{figure}

In contrast to the frequency-domain analysis of Fourier transforms, \emph{wavelets} provide a \emph{time--frequency} analysis. In particular, wavelets are a finite-support basis with a multiresolution structure, i.e., basis functions are divided into groups with different resolutions---some focus on ``local'' function values at very short timescales, while others capture more ``global'' structures at longer timescales. 
In the following, we explain the idea of wavelet MRA with the simplest wavelet family---\emph{Haar wavelets}. 
A formal treatment covering all orthogonal wavelets is in \cref{app:wavelets}. 

Suppose we want to approximate a signal $f(t)$ over the time interval $[0, 1)$. 
The roughest approximation we can produce is $\hat{f}^{(0)}(t) \triangleq a_{0,0} \phi(t)$ where 
$
\phi(t) = \mathbf{1}(0 \leq t < 1)
$
and $a_{0,0} = \int_0^1 f(t) dt$ is the average value of $f$. 
We use superscript $0$ to indicate that this is the lowest resolution approximation of $f$ we make.
We can better approximate $f$ %
by dividing the unit interval in two and approximating $f$ as: %
$f(t) \approx \hat{f}^{(1)}(t) \triangleq a_{1,0}\phi(2t) + a_{1,1}\phi(2t-1)$ where $a_{1,0}=\int_0^{1/2} f(t) dt$ and $a_{1,1}=\int_{1/2}^1 f(t)dt$.

We can repeat this procedure of halving the intervals, rescaling, and translating $\phi$, to get finer approximations $\{\hat f^{(j)}\}_{j\in \mathbb{N}_0}$.
Each $\hat f^{(j)}$ is a linear combinations of compactly supported basis functions,
$\{\phi_{j,k}(t) \triangleq 2^{j/2} \phi(2^j t - k)\}_{k\in \mathbb Z}$,
with their resolution levels indexed by $j$:
\begin{align*}
    \hat{f}^{(j)}(t) = \sum_{k \in \mathbb{Z}} a_{j,k} \phi_{j,k}(t), \text{\; where } a_{j,k} = \inner{f}{\phi_{j,k}}.  
\end{align*}
For each level $j \in \mathbb{N}_0$,  the subspace $V_j \defeq \mathrm{span}(\{\phi_{j,k}\}_{k\in \mathbb{Z}})$
contains functions that are constant over intervals of length $1/2^j$.
In other words, \emph{basis functions in $V_j$ describe structures in $f$ no larger than the timescale of $\Delta t \sim 1/2^j$}. For sufficiently large $j$, $V_j$ has the capacity to approximate any continuous time series arbitrarily well. 

One may try to summarize or represent $f$ by collecting the coefficients $\{a_{j,k}\}_{k\in \mathbb{Z}}$ into a vector.
Though the coefficients altogether fully describe the approximation $\hat{f}^{(j)}$, each individual coefficient alone may be too local to be representative of structures in $f$.
Each $a_{j,k}$ only summarizes the value of $f$ within a $1/2^j$ interval, while patterns may occur over larger intervals.
We would need multiple $a_{j,k}$ to summarize these larger-scale structures.
Is there a way to produce coefficients each of which summarizes a structure at a different scale?

\prepar
\paragraph{Representing structure at disjoint resolutions.}
We can indeed produce this kind of representation by using tools from MRA.
In MRA, we repeatedly decompose $V_j$ into the sum of a lower-resolution subspace $V_{j-1}$ and its orthogonal complement $W_{j-1}$: $V_j = V_{j-1} \oplus W_{j - 1}$.
Since basis functions in $V_{j}$ and $V_{j-1}$ describe structures at scales coarser than $\Delta t \sim 1/2^{j}$ and $\Delta t \sim 1/2^{j-1}$, respectively, basis functions in $W_{j-1}$ represent structures \emph{exactly at} the $1/2^j$ scale, summarized by the basis coefficients $\{b_{j,k}\}_{k\in \mathbb{Z}}$.
Starting from some high-resolution level $J$ and repeating this process, we have 
\begin{align*}
    V_J = V_{J-1} \oplus W_{J-1} = V_0 \oplus W_0 \oplus \ldots \oplus W_{J-2} \oplus W_{J - 1}
\end{align*}
and, as visualized in \cref{fig:components},
\begin{align} \label{eq:wavelet-series}
f(t) \approx \hat{f}^{(J)}(t) = a_{0,0} \phi(t) + \sum_{j'=0}^{J-1} \sum_{k\in \mathbb Z} b_{j',k} \psi_{j',k}(t).
\end{align}
The basis functions $\{\psi_{j,k}\}$ are called Haar wavelets and $\phi$ is called their scaling function; %
see \cref{app:haar-wavelets} for their functional forms. 
The coefficients\footnote{$\{a_{j,k}\}$ and $\{b_{j,k}\}$ are called the approximation coefficients and detail coefficients in signal processing. Note that some sources like \citet{Lee2019PyWaveletsPythonPackage} call $V_J$ level 0, $V_{J-1}$ level 1, and so on.} $\{a_{0,0}\} \cup \{b_{0,k}\}_{k\in \mathbb{Z}} \cup \ldots \cup \{b_{J-1,k}\}_{k\in \mathbb{Z}}$ now summarize the structures of $f$ at multiple resolutions, ranging from $1/2^0$ to $1/2^{J-1}$.

\prepar
\paragraph{Computing the representation.}
Our original problem of summarizing the multiresolution structures of $f$ then comes down to computing the wavelet basis coefficients $a_{0,0}, \{b_{j,k}\}$ of the approximation $\hat{f}^{(J)} \in V_J$.
See \cref{app:haar-wavelets} for how to compute these coefficients for Haar wavelets.
In general, we can efficiently and recursively compute these coefficients for any wavelet family using the discrete wavelet transform (DWT; see \cref{app:dwt}). 

In \cref{app:other-exps}, 
we illustrate the representational power of wavelet transforms.  In particular, we consider a raw audio waveform capturing 1 second of a recording at a sampling rate of 16,384.  We use a 10-level wavelet tree with a total of 2068 coefficients used to reconstruct the audio signal.  The wavelet transform is able to ``memorize'' many of the important patterns of the audio signal over this long sequence.  This representational power motivates our \ourlayer\ outlined in \cref{sec:multiresnet}.

\section{Sequence Modeling with Multiresolution Convolutions}
\label{sec:multiresnet}
\begin{figure*}[t]
\includegraphics[width=0.85\linewidth]{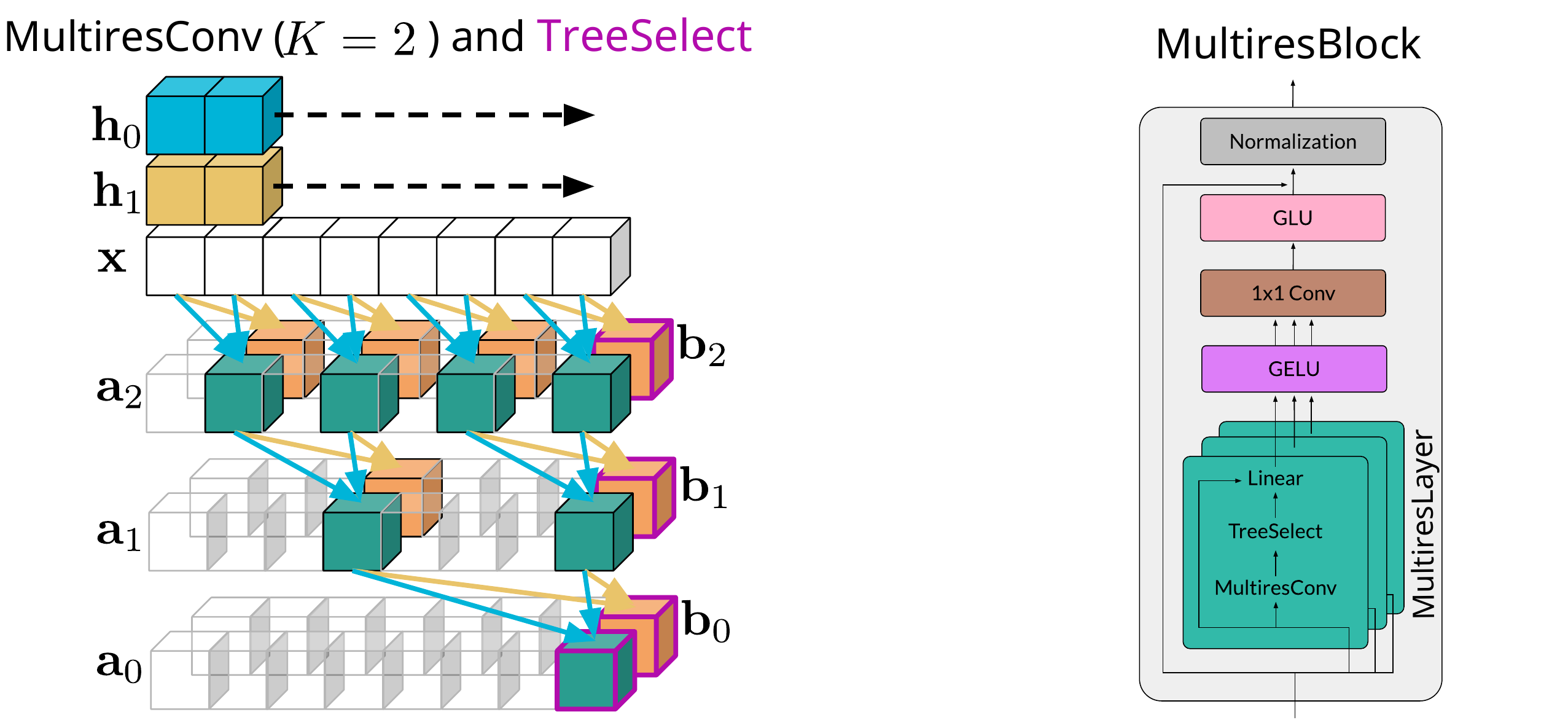}
\centering
\vspace{-4mm}
\caption{(Left) \ourconv\ consists of a sequence of dilated convolutions that share the same filters $\mathbf h_0, \mathbf h_1$ across all levels. Here we illustrate \textsc{TreeSelect} with the ``resolution fading'' strategy of keeping the right-most coefficient at each level of $\mathbf{a}_0, \mathbf{b}_{0:J-1}$, 
as indicated by magenta outlines. 
(Right) A schematic of our \treeblock\ architecture. Each channel of the input sequence is processed independently in the \ourlayer. 
The 1x1 convolution mixes the information across channels. 
We stack multiple \treeblock s\  to build our deep sequence models.} %
\vspace{-4mm}
\label{fig:multires-architecture}
\end{figure*}
We leverage the computation structure of DWT to construct a multiresolution memory for sequences. 
Given a sequence $\mathbf x \in \reals^{N}$ representing a discretely sampled signal, %
the DWT can be implemented by the following recursive computations for $\mathbf a_0$ and $\mathbf b_{0:J-1} \triangleq (\mathbf b_0, \mathbf b_1, \ldots, \mathbf b_{J-1}) $ starting from $\mathbf a_{J}(n) = \mathbf x(n)$\footnote{To deal with finite sequences, we pad $\mathbf{a}_{j+1}$ with zeros on the left before each iteration of the recursion, if necessary, to ensure that no elements of $\mathbf{a}_{j+1}$ are excluded. }:
\begin{align*}
\mathbf a_{j}(n) \triangleq a_{j,n} &= \sum_{k=0}^{K - 1} \mathbf a_{j+1}(2n + k) \mathbf h_0(k), \\
\mathbf b_{j}(n) \triangleq b_{j,n} &= \sum_{k=0}^{K - 1} \mathbf a_{j+1}(2n + k) \mathbf h_1(k), 
\end{align*}
where the filters $\mathbf h_0, \mathbf h_1 \in \reals^{K}$ are determined by the class of wavelets. 
For Haar wavelets, we have $\mathbf h_0 = (\frac{\sqrt{2}}{2}, \frac{\sqrt{2}}{2})$ and $\mathbf h_1 = (\frac{\sqrt{2}}{2}, -\frac{\sqrt{2}}{2})$. 
To decouple the underlying computation from the choice of filters, 
we define this procedure composed of convolution and downsampling at multiple scales as the \emph{multiresolution convolution} operation:
\begin{align}
\mathbf a_0, \mathbf b_{0:J-1} = \textsc{MultiresConv}(\mathbf x, \mathbf h_0, \mathbf h_1, J).
 \label{eq:multires-conv}
\end{align}
Recall, the coefficients $\{\mathbf a_{0}(n)\}$ and $\{\mathbf b_{j}(n)\}$ 
serve as a multiresolution representation of $\mathbf x$.  %
When the filter values come from wavelets, we can perfectly reconstruct the original sequence $\mathbf x$ from the coefficients $\mathbf a_0, \mathbf b_{0:J-1}$ by inverting the recursive procedure.
In other words, this procedure is powerful enough to give us \emph{perfect} memory of the past, summarized by the coefficients.

Instead of setting the filters to fixed values, however, we propose to use \ourconv\ as a building block for sequence models by letting $\mathbf h_0, \mathbf h_1$ be learnable.
Learning the filters allows us to go beyond hand-designed wavelets while still keeping the multiresolution structure in our computation. 
Although we may lose the ability to perfectly reconstruct the input, we will see in our experiments that we gain significant predictive improvements in return.

\subsection{A general sequence modeling layer}

Most competitive sequence models have the structure of mapping a sequence $\mathbf x\in \reals^N$ to another sequence $\mathbf y \in \reals^N$. 
The output at each timestep should 1) be relevant to the input taken in at that step and 2) capture potentially long-range dependencies from previous timesteps. 
This can be achieved by brute-force enumeration and comparison~\citep{vaswani2017attention}, which has quadratic complexity with respect to the input length. 
One can avoid quadratic scaling by maintaining a fixed-length memory of the past, as is done in recurrent models.  
However, a fixed-length memory will lose information over time, so we must ensure that the memory always retains the most relevant historical information.%

We use the \textsc{MultiresConv} operation of \cref{eq:multires-conv} as a memory mechanism to capture historical patterns at multiple resolutions. 
The key idea is to perform \textsc{MultiresConv} on the sequence up to the current timestep $t$, $\mathbf x(0:t)$, and \emph{select} a relevant subset of the representation coefficients $(\mathbf a^{t}_{0}$, $\mathbf b_{0:J-1}^{t})$ to form the memory vector $\mathbf z_t \in \reals^M$ at step $t$: 
\begin{align}
\mathbf a_0^{t}, \mathbf b_{0:J-1}^{t} &= \textsc{MultiresConv}(\mathbf x(0:t), \mathbf h_0, \mathbf h_1, J), \label{eq:seq-multiresconv} \\
\mathbf z_t &= \treeselect((\mathbf a_0^{t}, \mathbf b_{0:J-1}^{t}), M).\label{eqn:treeselect}
\end{align}
The $t$-step output $\mathbf y(t)$ is generated from the input at this step and the memory through a linear transform: 
\begin{align*}
\mathbf y(t) = \mathbf w^\top \begin{bmatrix}
\mathbf z_t \\
\mathbf x(t) 
\end{bmatrix}.
\end{align*} 
We repeat this procedure for $t=0, 1, \dots, N-1$ to get the entire output sequence $\mathbf y \in \mathbb{R}^N$. 
This architecture is schematically depicted in \cref{fig:multires-architecture} and
we discuss the coefficient selection process \treeselect\ in \cref{sec:memory}.

Interestingly, when $J$ is chosen to be the same for all $t$, the \ourconv s across time can be collectively viewed as multiple layers of dilated causal convolutions, similar to those in WaveNet~\citep{oord2016wavenet}. 
Thus, all computations here can be done in time linear in the sequence length and are massively parallelizable on modern hardware that implement convolutions. 
By default, we set
\begin{align} \label{eq:J-value}
    J = \left\lceil\log_2\Big(\frac{N - 1}{K - 1} + 1\Big)\right \rceil
\end{align}
such that $\mathbf{a}^t_0$ has a single element $a_{0,0}^t$ and its receptive field covers the whole $x(0:t)$ for any $t \leq N$.

\subsection{The memory mechanism}\label{sec:memory}
Ideally, we want to keep the entire sequence of $\mathbf a_0^{t}, \mathbf b_{0:J-1}^{t}$ for $t=0,\ldots, N$.
However, 
storing all of these coefficients %
requires space quadratic in $N$. Thus we need a way to \emph{select} coefficients from our \ourconv\ tree that best summarizes the history. 
We call this operation \treeselect.

We studied two strategies for determining which part of the \textsc{MultiresConv} outputs to include in the memory:
\begin{enumerate}[(1)]
    \item \textbf{Uniform over time:} Traditional usage of wavelets in signal compression and denoising thresholds out the time-localized, high-frequency components larger than a particular scale $j$. 
    Here, we can do the same by keeping $\mathbf a^{(t)}_0$ and $\mathbf b^{(t)}_{0}, \ldots, \mathbf b^{(t)}_j$ for $j$ below a threshold. 
    \item \textbf{Resolution fading:} Another strategy is to gradually increase our approximation resolution as time approaches $t$. 
    This can be achieved by keeping the entries of $\mathbf a^{(t)}_0$ and $\{\mathbf b^{(t)}_{j}\}$ at the highest index at each resolution level. We visualize this in \cref{fig:components,fig:multires-architecture}.
    This approximation allows us to retain the coefficients that are the most important for reconstructing the recent input values.
\end{enumerate}
Although our ablation study of \cref{sec:exp-ablation} did not highlight statistically significant differences between %
the two strategies, 
we use resolution fading in all of our experiments to bias the model to focus on memorizing the most recent information and for a simpler implementation---we provide an example PyTorch implementation %
in \cref{app:torch-impl}.
Similar techniques of upweighting recent history proved critical in state-space models~\citep{gu2020hippo,gu2021efficiently}.

In addition to the methods outlined above, an obvious next choice would be to dynamically select the relevant multiresolution components for the task at hand, perhaps via %
attention %
or $L_1$ regularization.  We leave this as future work.

\subsection{Multiple channels and the mixing layer}

The above only describes a mapping between single-channel sequences. 
To deal with an input $\mathbf x \in \reals^{d\times N}$ with $d$ channels, we adopt a simple strategy of stacking $d$ independent layers, %
which can be efficiently implemented using \emph{depthwise} causal dilated convolutions. 
The $d$ output sequences, after passing through a nonlinear activation function, 
are then mixed by a $1\times 1$ convolution. 
Here, we use the GELU activation~\citep{hendrycks2016gaussian}.
We note that this use of depthwise convolution followed by a mixing across channels has proven successful in a number of sequence modeling practices~\citep{wu2018pay,gu2021efficiently}.

\subsection{Stacking into deep models}

We wrap the sequence modeling layer and the mixing layer in a residual network block~\citep{he2016deep} with gated linear activations~\citep{dauphin2017language} and identity skip connections (see \cref{fig:multires-architecture}). 
We observed that the gating mechanism improves our model performance, which is consistent with observations made in a number of other sequence models~\citep{dauphin2017language, van2016conditional, oord2016wavenet, gu2021efficiently}. 
We then stack multiple residual blocks into a deep model. 
Depending on the task, layer normalization~\citep{ba2016layer} or batch normalization~\citep{ioffe2015batch} is applied after each residual block. 
We refer to this model as \ourmodel\ in our experiments.

\subsection{Optimization and regularization} 
We use the Adam optimizer with default hyperparameters and decoupled weighted decay~\citep{loshchilov2018decoupled}. 
Dropout is applied after the GELU and gated linear activation functions whenever overfitting is observed.

\section{Related Work}
\label{sec:related-work}

There has been significant activity building convolutional neural networks for sequential data, inspired by applications in modeling audio~\citep{waibel1989phoneme,oord2016wavenet} and natural language~\citep{collobert2011natural,kalchbrenner2016neural}. 
Recent work aims to improve the performance of such models through the use of gating units~\citep{dauphin2017language}, depthwise convolution~\citep{kaiserdepthwise}, input-dependent weights~\citep{shen2018learning,wu2018pay}, weight tying across depth~\citep{bai2019trellis}, and adaptive filter size~\cite{romero2022flexconv}. 
Our \textsc{MultiresConv} shares many ingredients of these works, but is unique in two aspects. 
First, our method explicitly defines memory units $\mathbf z_t$,  %
which are important for modeling long-range dependencies. 
Second, we have better theoretical underpinnings to our \textsc{MultiresConv} since it collapses to the standard DWT when the filters are specified as wavelet filters.

Among all convolutional sequence models, the closest to our work is the WaveNet architecture proposed in \citet{oord2016wavenet}. 
As shown in \cref{fig:multires-architecture}, when the filter size is 2, our computation graph for $a_j$s shares the same connection pattern as WaveNet. 
Both models are implemented with tree-structured dilated causal convolutions. 
However, there are three main differences between the two models: (1) our explicit memory construction via \treeselect\ creates a distinct skip-connection structure compared to WaveNet; (2) our model uses the same filters for all timescales; (3) we do not use nonlinear activation functions between convolutions at different timescales. 
As a result, our architecture is simpler and uses significantly fewer parameters than WaveNet. 
Additionally, the link we establish between wavelets and tree-structured dilated causal convolutions offers the first principled justification for the effectiveness of WaveNet in modeling raw audio waveforms, an exemplary case of lengthy sequences with multiscale structure. 

Recurrent neural networks and their linear variants (e.g., state-space models) also explicitly maintain an memory of the past.  
Models like S4~\citep{gu2021efficiently} are particularly relevant since they can be implemented as convolutions thanks to the linear recurrence. 
However, the convolution kernel for such models is as long as the input sequence, and efficiently computing these kernels relies on sophisticated parameterization and approximation techniques.
Although this issue can be circumvented by some recent advances~\citep{gupta2022diagonal,guparameterization,smith2022simplified}, initializing these models still requires special effort. 
The initialization mechanism shared by many of them, called HiPPO~\citep{gu2020hippo}, aims to memorize historical data via projection to orthogonal polynomials. 
Although this is related to  wavelets (i.e., a different basis in function space), we see in \cref{tab:ablation} that our method is insensitive to initialization---standard initialization schemes~\citep{glorot2010understanding} perform equally well as %
wavelet initialization. Some empirical motivation for the benefits of the wavelet basis over the HiPPO polynomial basis is given in \cref{app:other-exps} for an audio signal reconstruction task.

\section{Experiments}

We empirically test our model on classification and density estimation tasks that involve images, symbolic sequences, and physiological signals. 
Unless otherwise mentioned, we used standard Xavier uniform initialization for our \ourconv\ filters. 
All experimental details are presented in \cref{app:exp-details}. 
PyTorch code can be found in \url{https://github.com/thjashin/multires-conv}.

\begin{table}[t]
    \centering
    \vspace{-0.08in}
    \caption{Performance of pixel-level sequential classification on the sCIFAR dataset. %
        Bold indicates the best performing model and underline the second best. Results are taken from either the citation or \citet{hasani2022liquid}.}
    \label{tab:scifar}
        \vskip 0.1in
        \begin{small}
            \begin{tabular}{lccc}
                \toprule
                Model & Accuracy (\%) \\
                \midrule
                \emph{Attention}: & \\
                Transformer \citep{trinh2018learning} & 62.2 \\
                \midrule
                \emph{RNN}: & \\
                LSTM \citep{hochreiter1997long} & 63.01 \\
                r-LSTM \citep{trinh2018learning} & 72.2 \\
                UR-GRU \citep{gu2020improving} & 74.4 \\
                HiPPO-RNN \citep{gu2020hippo} & 61.1 \\
                LipschitzRNN \citep{erichson2021lipschitz} & 64.2 \\
                \midrule
                \emph{State Space Models}: & \\
                S4 \citep{gu2021efficiently} & 91.80 \\
                S4D \citep{guparameterization} & 90.69 \\
                S5 \citep{smith2022simplified} &  90.10 \\
                Liquid-S4 \citep{hasani2022liquid} & \underline{92.02} \\
                \midrule
                \emph{Convolution}: & \\
                TrellisNet \citep{bai2019trellis} & 73.42 \\
                CKConv \citep{romerockconv} & 63.74 \\
                FlexConv \citep{romero2022flexconv} & 80.82 \\
                \ourmodel~(Ours) & \textbf{93.15} \\
                \bottomrule
            \end{tabular}
        \end{small} 
\end{table}

\subsection{Pixel-level sequential image classification}
\label{sec:exp-sCIFAR}

We first consider image classification tasks where images are treated as a 1D sequence of pixels. 
The models are not allowed to use any 2D bias from the  image. 
Therefore, the model must be able to capture patterns at multiple different timescales, including pixels that are near in the original image but far from each other in its sequence representation. 

We evaluate our model on the Sequential CIFAR-10 dataset, which has long been used as a standard benchmark for modeling long-range dependencies in RNNs. 
We use the standard train and test split of the CIFAR-10 dataset and leave out 10\% of the training set as the validation set. 
For the classification task, we perform a mean-pooling of all timesteps on the output sequences ($\mathbf{y}_{[\text{batchsize},256, N]}\, \rightarrow \, \overline{\mathbf{y}}_{[\text{batchsize},256]}$) and pass the result into a fully-connected layer to generate the class logits. 
We report the test accuracy from the model that has the highest validation accuracy. 

The results are reported in \Cref{tab:scifar}. 
As we see, our model yields state-of-the-art performance (best test accuracy) on this benchmark sequence classification task, outperforming many recent strong competitors including transformers~\citep{vaswani2017attention}, RNNs, state space models, and other convolutional models. 
In particular, \ourmodel\ outperforms previous purely convolution-based models by more than 10 percentage points. 

Surprisingly, our model achieves this new performance benchmark with an almost embarrassingly simple architecture comprised primarily of dilated causal convolutions with length-2 filters shared between tree levels, and levels connected by linear links. In particular, our best model uses 10 \textsc{MultiresBlock}s, each containing a \textsc{MultiresConv} layer that amounts to 10 layers of tied-weight dilated causal convolutions. 
Together, that amounts to 100 layers of dilated causal convolutions, but the total number of parameters is only 1.4M. 
In contrast, the best S4 model uses around 7.9M parameters~\citep{gu2021efficiently}, but still underperforms our 5x smaller model (and likewise relies on fancy initializations).

\subsection{Hierarchical reasoning on symbolic sequences}
\label{sec:listops}

\begin{table}[t] 
    \centering
    \vspace{-0.08in}
    \caption{Performance of predicting outcomes of list operations in the long ListOps dataset of~\citet{tay2021long}.  Bold indicates the best-performing model and underlines the second best. *SGConv is built to mimic the global convolution interpretation of S4. Results are taken from either the citation or \citet{hasani2022liquid}.}
     \label{tab:listops}
    \vskip 0.1in
    \begin{small}
\begin{tabular}{lc}
    \toprule
    Model & Accuracy (\%)  \\
    \midrule
    \emph{Attention}: & \\
    Local Attention \citep{tay2021long} & 15.82 \\
    Linear Trans. \citep{katharopoulos2020transformers}& 16.13  \\
    Linformer \citep{wang2020linformer} & 16.13 \\
    Sparse Transformer \citep{child2019generating} & 17.07 \\
    Performer  \citep{choromanski2020rethinking} & 18.01  \\
    Transformer \citep{vaswani2017attention} & 36.37  \\
    Sinkhorn Transformer \citep{tay2020sparse} & 33.67  \\
    FNet \citep{lee2022fnet} & 35.33  \\
    Longformer \citep{beltagy2020longformer}  & 35.63 \\
    BigBird \citep{zaheer2020big} & 36.05	 \\
    Nyströmformer \citep{xiong2021nystromformer} & 37.15  \\
    Luna-256 \cite{ma2021luna} & 37.25  \\
    Reformer \citep{kitaev2020reformer} & 37.27  \\
    H-Transformer-1D \citep{zhu2021h} & 49.53 \\
    \midrule
    \emph{State Space Models}: & \\
    S4 \citep{guparameterization} & 59.60  \\
    DSS \citep{gupta2022diagonal} & 57.6  \\
    S4D \citep{guparameterization} & 60.52  \\
    S5 \citep{smith2022simplified} & \underline{62.15} \\
    Liquid-S4~\citep{hasani2022liquid} & \textbf{62.75}  \\
    \midrule
    \emph{Convolution}: & \\
    CDIL \citep{cheng2023classification} & 44.05 \\
    SGConv* \citep{li2022makes} & 61.45 \\
    \ourmodel\ (Ours) & \textbf{62.75} \\
    \bottomrule
\end{tabular}
    \end{small}
\end{table}

\begin{table*}[h]
    \vspace{-0.08in}
    \caption{AUROC for ECG multi-label/multi-class classification on the PTB-XL dataset. Bold indicates the
        best performing model and underline the second best. Results for other models taken from \citet{zhang2023effectively} and \citet{Strodthoff2021DeepLearningECG}.}
    \vskip 0.1in
    \centering
    \begin{small}
        \begin{tabular}{@{}lcccccc@{}}
            \toprule
            Model (AUROC) & All & Diag & Sub-diag & Super-diag & Form & Rhythm\\
            \midrule
            \ourmodel \ (Ours)& \textbf{0.938} & \underline{0.939} & \textbf{0.934} & \textbf{0.934} & \underline{0.897} & \underline{0.975} \\
            Spacetime \citep{zhang2023effectively} & \underline{0.936} & \textbf{0.941} & \underline{0.933} & 0.929 & 0.883 & 0.967 \\
            S4 \citep{gu2021efficiently} & \textbf{0.938} & \underline{0.939} & 0.929 & \underline{0.931} & 0.895 & \textbf{0.977} \\
            InceptionTime \citep{IsmailFawaz2020InceptionTimeFindingAlexNet} & 0.925 & 0.931 & 0.930 & 0.921 & \textbf{0.899} & 0.953 \\
            LSTM \citep{hochreiter1997long} & 0.907 & 0.927 & 0.928 & 0.927 & 0.851 & 0.953 \\
            Transformer \citep{vaswani2017attention} & 0.857 & 0.876 & 0.882 & 0.887 & 0.771 & 0.831 \\
            Wavelet features \citep{Strodthoff2021DeepLearningECG} & 0.849 & 0.855 & 0.859 & 0.874 & 0.757 & 0.890 \\
            \bottomrule
        \end{tabular}
        \label{tab:ptbxl}
    \end{small}
\end{table*}

In order to test our model's capability of reasoning about hierarchical structures, we conduct experiments on the long ListOps dataset from \citet{tay2021long}. 
The dataset consists of sequences that represent a composition of multiple list operations including \texttt{MAX, MEAN, MED} (median) and \texttt{SM}~(sum and mod). 
For example, the input can be
\begin{align*}
\texttt{[MAX 1 [MAX 2 3 ] 5 6 [MIN 7 8 ] ]}.  
\end{align*}
The output in this case is 7. 
This is a far shorter version of the examples in the dataset which have length up to 2048. 
The prediction is formulated as a ten-way classification problem. 
We use the train, validation, and test split specified by \citet{tay2021long}. 
Following the practice in prior work, we pad the sequences to maximum length (2048) to form minibatches, and average over only the actual inputs when pooling over output sequences. Our model for this experiment has 12 \textsc{MultiresBlock}s. 
We found a larger filter size (4) is beneficial for this task~\footnote{This corresponds to switching from the tree structure of the Haar DWT to those of more general wavelets, such as the Daubechies wavelets~\citep{akansu2001multiresolution}.}. 

Results are reported in \Cref{tab:listops}. 
\ourmodel\ achieves the overall %
best result, outperforming all attention-based models, S4 and its diagonal variants. 
It is only %
matched by Liquid-S4 which additionally introduced inner-product structures between inputs into state space models. 

\ourmodel\ is the first model with small kernel convolutions to achieve competitive performance with state space models on the long ListOps dataset. 
Although SGConv~\citep{li2022makes} is advertised as a convolutional neural network, the model is built is to mimic the global convolution interpretation of S4. 
Therefore, it shares the same limitation of a sophisticated parameterization and relies on an FFT for computing the large kernel convolution.

\subsection{Classifying physiological signals}
PTB-XL \citep{Wagner2020PTBXLLargePublicly} is a publicly available dataset of electrocardiogram (ECG) time series. %
The dataset contains 21,837 12-lead ECG recordings for 10 seconds from 18,885 patients.
Each recording is labeled with at least one of the 71 ECG labels from the SCP-ECG standard.
The dataset is partitioned into six subsets: ``all'', ``diagnostic'', ``diagnostic subclass'', ``diagnostic superclass'', ``form'', and ``rhythm'', each containing a different subset of the full 71 labels.
``Diagnostic superclass'' is a multi-class classification task while the others are multi-label classification tasks.

We use the 100Hz version of the dataset, where each time series has 1K timesteps and 12 channels.
We use the train, validation, and test split specified by \citet{Strodthoff2021DeepLearningECG}. 
We transferred the architecture and learning rate from our Sequential CIFAR-10 experiments. 
We tuned the dropout rate on the validation set of ``diagnostic superclass'' and used the same setting for all subsets afterwards. 

We show the results in \cref{tab:ptbxl}.
On all six subsets of PTB-XL, our model either attains %
the best or second-best AUROC out of all models presented.
It significantly outperforms a neural network trained on fixed wavelet features~\citep{Strodthoff2021DeepLearningECG}, 
showing the benefits of combining multiresolution computation with end-to-end learning.

\begin{table}[t]
    \centering
    \vspace{-0.08in}
    \caption{Autoregressive generative modeling on CIFAR-10 dataset. Results are reported as bits per dimension (bpd) (lower is better).  
        Bold indicates the best-performing model and underlines the second best. Results are taken from the citation.}
    \label{tab:ar-cifar}
        \setlength{\tabcolsep}{5pt}
        \vskip 0.1in
        \begin{small}
            \begin{tabular}{lcc}
                \toprule
                Model & \#params & Test bpd. \\
                \midrule
                \emph{RNN + 2D bias}: \\
                PixelRNN \citep{van2016pixel} &  & 3.00 \\
                \midrule
                \emph{2D Convolution:} \\
                PixelCNN \citep{van2016pixel} & & 3.14 \\
                Gated PixelCNN \citep{van2016conditional} &  & 3.03 \\
                PixelCNN++ \citep{salimans2017pixelcnn} & 53M  & 2.92 \\
                \midrule
                \emph{2D Convolution + Attention:} \\
                PixelSNAIL \citep{chen2018pixelsnail} & 46M & 2.85 \\
                \midrule
                \emph{1D Attention:} & & \\
                Image Trans. \citep{trinh2018learning} &  & 2.90 \\
                \emph{2D Attention:} & & \\
                Sparse Trans. \citep{child2019generating} & 59M  & \textbf{2.80} \\
                \midrule
                \emph{State-Space Models + U-Net structure:} & & \\
                S4 \citep{gu2021efficiently} &  & 2.85 \\
                \midrule
                \emph{Convolution (no 2D bias)}: \\
                \ourmodel~(Ours) & 38M &  \underline{2.84} \\
                \bottomrule
            \end{tabular}
        \end{small}\vspace{-0.1cm}
\end{table}

\subsection{Autoregressive generative modeling}

In this experiment, we go beyond classification and evaluate our sequence modeling layer on autoregressive generative modeling for CIFAR-10. 
A detailed description of the settings is provided in \cref{app:ar-details}. 
We used 48 \ourlayer s with 512 channels and filter size 4. 
As the nature of task requires more parameters than the former experiments, we added an additional $1\times 1$ convolutional layer after the GLU activation in the residual block. 

\begin{table*}[h]
    \vspace{-0.08in}
    \caption{Results of ablation study performed on the long ListOps dataset. We used a smaller model than the one in \cref{sec:listops} and report the average performance over three random seeds. }
    \label{tab:ablation}
    \vskip 0.1in
    \centering
    \begin{small}
        \begin{tabular}{@{}lcccccr}
            \toprule
            ID & Filters & Initialization & Memory mechanism & Filter size & MultiresConv depth & Accuracy \\
            \midrule
            1 & Fixed & wavelet & Fading &  2  & 7 & 50.15$\pm$0.60 \\
            2 & Trainable & wavelet & Fading & 2 & 7 & 52.08$\pm$0.35 \\
            3 & Trainable & Xavier & Fading & 2 &  7 & 51.70$\pm$0.16 \\
            4 & Trainable & Xavier & Fading & 2 & 11 & 61.07$\pm$0.26 \\
            5 & Trainable & Xavier & Uniform & 2 & 7 & 51.58$\pm$0.57 \\
            6 & Trainable & wavelet & Fading & 4 & 7 & 59.23$\pm$0.13 \\
            \bottomrule
        \end{tabular}
    \end{small}
\end{table*}

We hold out 2K examples from the training set for validation and report test negative log-likelihood (in bits per dimension) with best validation performance.
We benchmark our results with a number of previously reported results %
in \cref{tab:ar-cifar}. 
With no 2D bias and significantly fewer parameters, our model achieved 2.838 bits per dimension on the test set, outperforming 2D convolution-based models by a large margin and even surpasses their improved variant (PixelSNAIL) that leverages both convolution and attention. 
Our model also outperforms S4, which did not report details of its architecture, but mentioned they used downsampling and a U-Net structure. 
Our model has no downsampling or U-Net structure, demonstrating its memory efficiency and the ability to capture long-range dependencies. 
It also beats Image Transformer with 1D attention, and is only outperformed by Sparse Transformer, which additionally leveraged 2D biases and used significantly more parameters. 

\subsection{Ablation study}
\label{sec:exp-ablation}

We reuse the long ListOps dataset from \cref{sec:listops} and conduct six experiments to study the effect of model architecture, hyperparameters and initialization schemes. 
Each experiment consists of three independent runs with different random seeds.
The results are reported in \Cref{tab:ablation}. 

\prepar
\paragraph{Fixed wavelet filters vs. trainable filters. } By comparing Experiment 1 and 2, we can see that training the wavelet filters improves the performance, which justifies our design choice of decoupling the \textsc{MultiresConv} operation from the wavelet transform. 

\prepar
\paragraph{Sensitivity to initialization. } 
Next, we examine the effect of initializing the filters as wavelet filters. 
From Experiment 2 to 3, we switched from wavelet initialization to standard Xavier initialization of neural networks~\citep{glorot2010understanding}, fixing all the other parts of the model. 
We do not observe statistically significant differences between the approaches. 
This demonstrates the advantage of our model over S4-related methods that require careful initialization. 

\prepar
\paragraph{Memory mechanism. } We do not notice a statistically significant difference between uniform and resolution fading, though resolution fading provides a simpler implementation.

\prepar
\paragraph{Importance of receptive fields. }
Finally, we show that we can significantly improve the performance of this model by increasing either the filter size (Experiment 2 vs. 6) or the depth of the $\textsc{MultiresConv}$ (Experiment 3 vs. 4). 
We believe this is because both changes increase the receptive field size of the $\textsc{MultiresConv}$ operation, which is particularly important for reasoning tasks like ListOps.

\section{Conclusion}
We presented \ourlayer\ for robust and efficient memorization of long-term patterns in sequential data sources.  
It takes inspiration from the multiresolution analysis (MRA) literature, building on wavelet decompositions, to memorize patterns occurring at multiple timescales.  In particular, our memory is generated by \emph{multiresolution convolution}s, %
implemented as dilated causal convolutions with \emph{learned} filters shared between tree levels that are connected via purely linear operations.  To create the memory, all multiresolution values may be maintained, or more emphasis can be placed on more recent time points by leveraging the time-localized nature of wavelet transforms.

The resulting \ourmodel\ garners the computational advantages of convolutional networks while being defined by dramatically fewer parameters than competitor models, all while achieving state-of-the-art performance in a number of benchmark sequence modeling tasks.  
These experiments demonstrate the portability of our multiresolution memory structure to a number of tasks, even in cases where a given task may not intuitively be viewed in a multiscale fashion (e.g., syntax tree parsing in ListOps).

By taking inspiration from the wavelet literature, we built an effective convolutional layer with dramatically fewer parameters without taking a performance hit.  The principled underpinnings of the \textsc{MultiresConv} ensure it possesses a configuration with strong reconstruction capabilities (e.g., when our filters equal the wavelet filters); however, as we showed, predictive performance can be improved by learning the filters.

Another potential benefit of starting from the wavelet framework is the ability to leverage that vast literature in that domain for future modeling advances.  In particular, we plan to explore the utility of \ourconv\ in representation learning %
and long-term forecasting. For representation learning, we can consider the structure of \emph{shift-invariant wavelet transforms}~\cite{kingsbury1998dual,selesnick2005dual} to target representations that are invariant to shifts of the input signals.  For example, we may want to cluster individuals with similar ECG signals even if the key signatures are shifted relative to one another.  Wavelets may also be extended to image analysis, enabling video analysis in our sequential setting.

\section*{Acknowledgements}

This work was supported in part by AFOSR Grant FA9550-21-1-0397, ONR Grant N00014-22-1-2110, the National Science Foundation under grant 2205084, and the Stanford Institute for Human-Centered Artificial Intelligence (HAI). EBF is a Chan Zuckerberg Biohub – San Francisco Investigator.
KAW was partially supported by Stanford Data Science as a Stanford Data Science Scholar.

\bibliography{ref}
\bibliographystyle{icml2023}

\clearpage
\appendix
\onecolumn

\section{Background: Wavelet Theory}
\label{app:wavelets}

In this section, we provide a formal treatment of the background on wavelet theory, expanding upon the introduction presented in \cref{sec:wavelet-background}.
We begin by going through the derivation of Haar wavelets using the multiresolution analysis (MRA) framework. 
We then demonstrate how the same framework can be leveraged to derive a wide array of orthogonal wavelets.   

\subsection{Haar wavelets}
\label{app:haar-wavelets}

We first introduce the definition of Haar \emph{scaling function}. 
The scaling function is also known as the mother wavelet.

\begin{definition}[Haar scaling function]
The Haar scaling function is defined as
\begin{align*}
    \phi(t) = \mathbf{1}(0 \leq t < 1).
\end{align*}
\end{definition}
Let $V_0$ be the space spanned by translations of the scaling function:
\begin{align*}
    V_0 \triangleq \left\{s_0\middle | s_0(t) = \sum_{k\in \mathbb{Z}} a_{0,k} \phi(t - k), a_{0,k} \in \reals\right\}. 
\end{align*}
It contains functions that are constant over intervals of length 1.
Next, we construct a larger space $V_1$ that subsumes $V_0$ by dividing the unit interval into halves:
\begin{align*}
    V_1\triangleq \left\{s_1\middle | s_1(t) = \sum_{k \in \mathbb{Z}} a_{1,k} \sqrt{2}\phi(2t - k), a_{1,k} \in \reals\right\}. 
\end{align*}
The coefficient $\sqrt{2}$ is introduced such that the basis function $\sqrt{2}\phi(2t - k)$ is normalized, i.e., $\int_{\reals} (\sqrt{2}\phi(2t - k))^2 \;dt = 1$. 
These functions are constant over intervals of length $1/2$. 
We could repeat such a process and create a sequence of function spaces for $j \geq 0$:
\begin{align} \label{eq:vj}
    V_j\triangleq \left\{s_j\middle | s_1(t) = \sum_{k \in \mathbb{Z}} a_{j,k} \phi_{j,k}(t), a_{j,k} \in \reals\right\}, \text{ where } \phi_{j,k}(t) \triangleq 2^{j/2} \phi(2^j t - k). 
\end{align}
These function spaces come with increasing richness and expressive power:
\begin{align*}
    V_0 \subset V_1 \subset V_2 \subset \dots \subset V_{j}. 
\end{align*}
It is easy to see that $V_j$ with a sufficiently large $j$ will have the capacity to approximate any continuous time series $f$ with arbitrary precision. 
Let $\hat{f}^{(j)}(t) = \sum_{k\in \mathbb{Z}}a_{j,k}\phi_{j,k}(t)$ denote the approximation of $f$ in $V_j$. 
The best approximation is given by the coefficients
\begin{align} \label{eq:ajk}
    a_{j,k} = \int_{\reals} f(t) \phi_{j,k}(t)\; dt =  
    \int_{k/2^j}^{(k+1)/2^j} f(t)\; dt.  %
\end{align}
One can represent $f$ with the vector of coefficients $\{a_{j,k}\}_{k\in \mathbb{Z}}$. 
However, as we have explained in \cref{sec:wavelet-background}, each individual coefficient is too localized to capture patterns over larger intervals. 
One way to remedy this is applying MRA to generate representations that summarize patterns of $f$ at multiple different timescales, resulting in Haar wavelets.

The wavelet MRA has the following steps. 
First, observing that $V_{j-1} \subset V_j$, we decompose $V_j$ into the sum of $V_{j-1}$ and its orthogonal complement $W_{j-1}$. 
Repeating this process, we get  
\begin{align} \label{eq:vj-decomp}
    V_j = V_{j-1} \oplus	 W_{j - 1} = V_0 \oplus W_0 \oplus W_1 \oplus \dots \oplus W_{j-1}.  
\end{align}
Because the functions in $V_j$ and $V_{j-1}$ represent structures at timescales coarser than $1/2^j$ and $1/2^{j-1}$, respectively, 
we can expect the basis functions in $W_{j-1}$ uniquely represent structures at the $1/2^j$ timescale. 
The next step is to find an orthonormal basis for $W_{j-1}$. 
Theorem \ref{thm:wavelet} shows that this basis 
can be obtained from rescaling and translations of a function $\psi$, the collection of functions known as the \emph{Haar wavelets}. 

\begin{definition}[Haar wavelets] \label{def:haar-wavelets}
    The Haar wavelets are a family of functions defined as
    \begin{align}
        \psi_{j,k}(t) 
        =  2^{j/2}\psi(2^jt - k), \text{ where }\psi(t) \defeq \phi(2t) - \phi(2t - 1).
    \end{align}
\end{definition}

\begin{theorem}[Haar wavelets as orthonormal bases] \label{thm:wavelet}
    Let $V_j$ be defined as in \eqref{eq:vj} and $W_{j-1}$ be the orthogonal complement of $V_{j-1}$ in $V_j$. 
    Then, $\{\psi_{j,k}\}_{k\in \mathbb{Z}}$ is an orthonormal basis for $W_j$. 
\end{theorem}

\begin{proof}
    We first notice that $\psi_{j,k} \in V_{j+1}$ and $\inner{\psi_{j,k}}{\psi_{j',k'}} = \delta_{j,j'}\delta_{k,k'}$. 
    To show the result, we need to prove (1) $\psi_{j,k} \in W_j$, i.e., $\psi_{j,k}$ is orthogonal to any function in $V_j$, and (2) any function in $W_j$ can be represented with $\{\psi_{j,k}\}_{k\in \mathbb{Z}}$. 
    For the first requirement, we have that, for any $k'\in \mathbb{Z}$,
    \begin{align*}
        \inner{\psi_{j,k}}{\phi_{j, k'}} = \inner{\frac{1}{\sqrt{2}}(\phi_{j+1, 2k} - \phi_{j+1, 2k+1})}{\frac{1}{\sqrt{2}}(\phi_{j+1, 2k'} + \phi_{j+1, 2k'+1})} = \frac{1}{2}(\delta_{k,k'} - \delta_{k,k'}) = 0. 
    \end{align*}
    Therefore, $\psi_{j,k}$ is orthogonal to any function in $V_j$, which implies $\psi_{j,k} \in W_j$. 
To prove (2), we notice that for some $s_{j+1} = \sum_{k\in \mathbb{Z}} a_{j+1,k} \phi_{j+1,k} \in V_{j+1}$, $s_{j+1} \in W_j$ implies for any $k'\in \mathbb{Z}$,
\begin{align*}
    \inner{s_{j+1}}{\phi_{j,k'}} &= \sum_{k\in \mathbb{Z}} a_{j+1,k} \inner{\phi_{j+1,k}}{\phi_{j, k'}} \\
    &= \sum_{k\in \mathbb{Z}} a_{j+1,k} \inner{\phi_{j+1,k}}{\frac{1}{\sqrt{2}}(\phi_{j+1,2k'} + \phi_{j+1, 2k' + 1})} \\
    &= \frac{1}{\sqrt{2}}(a_{j + 1,2k'} + a_{j + 1,2k' + 1}) = 0.
\end{align*}
Therefore, we can write any $s_{j+1} \in W_j$ using $a_{j+1,2k+1} = -a_{j+1,2k}$ as 
\begin{align*} 
    s_{j+1} = \sum_{k\in \mathbb{Z}} a_{j+1,2k} (\phi_{j+1, 2k} - \phi_{j+1, 2k + 1}) = \sum_{k\in \mathbb{Z}} \sqrt{2} a_{j+1,2k} \psi_{j,k}. 
\end{align*}
Thus, any function in $W_j$ can be represented with $\{\psi_{j,k}\}_{k\in \mathbb{Z}}$. 
Combining (1) and (2) proves the original statement. 
\end{proof}

Now that we have the wavelet basis, we can rewrite the approximation of $f$ in $V_j$ to reflect the decomposition in \eqref{eq:vj-decomp}:
\begin{align} \label{eq:wavelet-expansion}
    f(t) \approx \hat{f}^{(j)}(t) = \sum_{k\in \mathbb{Z}} a_{0,k} \phi(t - k) + \sum_{j'=0}^{j-1} \sum_{k'\in \mathbb Z} b_{j',k'} \psi_{j',k'}(t), 
\end{align}
where $\{\phi(t - k)\}_{k\in \mathbb{Z}}$ is the basis of $V_0$ and $\{\psi_{j',k'}\}_{k'\in \mathbb{Z}}$ is the basis of $W_{j'}$ for $j' = 0, 1, \cdots, j - 1$. 
Therefore, we can summarize the multiresolution structure of $f$ using the coefficients $\{a_{0,k}\}_{k\in \mathbb{Z}}$ and $\{b_{0,k}\}_{k\in \mathbb{Z}} \cup \cdots \cup \{b_{j-1,k}\}_{k\in \mathbb{Z}}$.

\subsection{Generalization to all orthogonal wavelets}
\label{app:orth-wavelets}

The multiresolution analysis framework described in the previous section can be generalized beyond Haar wavelets
to all orthogonal wavelets. 
The starting point is that we can use a different scaling function than the Haar scaling function that satisfies the following equation:
\begin{align} \label{eq:orth-wavelet-father}
    \phi(t) = \sum_{k=0}^{K-1}\mathbf h_0(k)\sqrt{2}\phi(2t - k).
\end{align}
This condition guarantees that the associated function spaces satisfy $V_{j-1} \subset V_j$. 
The sequence $\mathbf h_0(\cdot) \in \mathbb{R}^K$ here is referred to as a (low-pass) \emph{filter}, with potentially different lengths for the various wavelets.
For example, the Haar wavelet filter is a 2D vector with $\mathbf h_0(0) = \mathbf h_0(1) = \frac{1}{\sqrt{2}}$. 
The Daubechies-4 wavelet~\citep{daubechies1988orthonormal} filter has 4 elements:
\begin{align*}
\left(\frac{1 + \sqrt{3}}{4\sqrt{2}}, \frac{3 + \sqrt{3}}{4\sqrt{2}}, \frac{3 - \sqrt{3}}{4\sqrt{2}}, \frac{1 - \sqrt{3}}{4\sqrt{2}}\right). 
\end{align*}
Similarly, we have the following generalized form of
the mother wavelet:
\begin{align} \label{eq:orth-wavelet-mother}
    \psi(t) = \sum_{k=0}^{K-1}\mathbf h_1(k)\sqrt{2}\phi(2t - k),
\end{align}
where $\mathbf h_1(\cdot)$ is known as the \emph{high-pass} filter.
Using the orthogonality between $\phi(t-k)$ and $\psi(t)$ for any $k$, one can show~\citep[see, e.g.,][]{strang1996wavelets}: 
\begin{align*}
    \mathbf h_1(k) = (-1)^k \mathbf h_0(K - 1-k),
\end{align*}
where $N$ is the length of the filter $\mathbf h_0$. %
For Haar wavelets, $\mathbf h_1(0) = \mathbf h_0(1) = \frac{1}{\sqrt{2}}$, $\mathbf h_1(1) = -\mathbf h_0(0) = -\frac{1}{\sqrt{2}}$. 
For Daubechies-4 wavelets, the corresponding $\mathbf h_1$ filter is
\begin{align*}
    \left(\frac{1 - \sqrt{3}}{4\sqrt{2}}, -\frac{3 - \sqrt{3}}{4\sqrt{2}}, \frac{3 + \sqrt{3}}{4\sqrt{2}}, -\frac{1 + \sqrt{3}}{4\sqrt{2}}\right). 
\end{align*}
In this work, except for ablation purposes, we do not fix the values of the filters ($\mathbf h_0, \mathbf h_1$) or enforce coupling between them.
Instead, we allow both filters to be learned via gradient descent. 
We also found that initializing the filters with Haar or Daubechies wavelets does not provide advantages over standard Xavier initialization.
This demonstrates the ability of the model to discover effective parameterizations through optimization.

\subsection{Discrete wavelet transform}
\label{app:dwt}

We have already shown that a continuous time series can be represented with a vector of coefficients that capture its structure at multiple different timescales. 
The remaining challenge is determining these coefficients from a discretely sampled signal $\mathbf x \in \mathbb{R}^N$. 
From \eqref{eq:ajk}, we know that each coefficient $a_{j,k}$ summarizes the average value of the time series in a $1/2^j$ interval. 
Therefore, if we choose a sufficiently large $J$, the coefficients $\mathbf{a}_{J}$ can be directly approximated with the discretely sampled $\mathbf x$. 
We can obtain the coefficients needed for the wavelet expansion \eqref{eq:wavelet-expansion} recursively from $\mathbf{a}_{J}$ using a process known as the discrete wavelet transform (DWT).

To derive the DWT recursion, we start from $s_{j+1} = \sum_{m} a_{j+1, m}\phi_{j+1,m} \in V_{j+1}$ and compute the coefficents for its decomposition in $V_j$ and $W_j$:
\begin{align*}
    s_{j+1} = \sum_{n} a_{j,n} \phi_{j,n} + \sum_{\ell} b_{j, \ell}\psi_{j,\ell}. 
\end{align*}
Due to the orthogonality of $\{\phi_{j, n}\}_{n\in \mathbb{Z}}$ and $\{\psi_{j, n}\}_{n\in \mathbb{Z}}$, 
\begin{align*}
    a_{j, n} &= \inner{s_{j+1}}{\phi_{j,n}} \\
    &= \inner{\sum_m a_{j+1,m} \phi_{j+1,m}}{\sum_k \mathbf h_0(k) \phi_{j+1, 2n + k}} \\
    &= \sum_{k} a_{j+1, 2n + k}  \mathbf h_0(k), \text{ and similarly,}\\
    b_{j, \ell} &= \inner{s_{j+1}}{\psi_{j,\ell}} = \sum_{k} a_{j+1, 2\ell + k}  \mathbf h_1(k). 
\end{align*}
The DWT is repeating the above computations from $j=J$ to $0$. 
We can see each step of this process is a convolution between $\mathbf a_{j+1}$ and the filter ($\mathbf h_0(-\cdot)$ or $\mathbf h_1(-\cdot)$), and then applying a downsampler ($\mathbf y(n) = \mathbf x(2n)$). 
Note that although the convolution is done after a negation of the filter $\mathbf h_0, \mathbf h_1$, this is exactly equivalent to a 1-dimensional convolutional layer in neural networks with the convolution kernel $\mathbf h_0$ or $\mathbf h_1$, where the convolution operation is standard convolution plus negation of the filter.

\section{MultiresLayer Implementation}
\label{app:torch-impl}

We provide an example PyTorch implementation of a \ourlayer\ with the "resolution fading" TreeSelect strategy (see \cref{sec:memory}) in \Cref{fig:torch-impl}.

\begin{figure*}[h]
\centering
\begin{minted}[fontsize={\fontsize{8.5}{9.5}\selectfont}]{python}
import math
from torch.nn.functional import pad, conv1d    

def multires_layer(x, h0, h1, w, depth=None):
    """
    Args:
        x: input of shape (batch_size, n_channels, sequence_length).
        h0, h1: convolution filters of shape (n_channels, 1, filter_size).
        w: weights of the linear layer after TreeSelect. Shape: (n_channels, depth + 2).
        depth: depth of MultiresConv, i.e., J.
    
    Returns:
        y: output of shape (batch_size, n_channels, sequence_length).
    """
    kernel_size = self.h0.shape[-1]
    if depth is None:
        depth = math.ceil(math.log2((x.shape[-1] - 1) / (kernel_size - 1) + 1))
    y = 0.
    a = x
    dilation = 1
    for i in range(depth, 0, -1):
        a = pad(a, (dilation * (kernel_size - 1), 0), "constant", 0)
        b = conv1d(a, h1, dilation=dilation, groups=x.shape[1])
        a = conv1d(a, h0, dilation=dilation, groups=x.shape[1])
        y += w[:, i:i+1] * b
        dilation *= 2
    y += w[:, :1] * a 
    y += w[:, -1:] * x
    return y
\end{minted}
\caption{PyTorch code for implementing a \ourlayer\ with the "resolution fading" TreeSelect strategy. } \label{fig:torch-impl}
\end{figure*}

\section{Experiment Details}
\label{app:exp-details}

\subsection{Sequential CIFAR-10 classification}

The inputs are sequences of length 1024 with three channels corresponding to RGB values. 
We rescaled and centered the inputs between $[-1, 1]$, and used a $1\times 1$ convolutional layer to encode them into 256 channels.
For this task, we used 10 \ourblock s, 
each with filter size 2, 256 channels, and the resolution fading memory mechanism. 
They were followed by a mean-pooling over timesteps and a linear layer to generate classification logits. 
We trained the network for 250 epochs with batch size 50.
We used the AdamW optimizer~\citep{loshchilov2018decoupled} with default hyperparameters and a weight decay rate 0.01. 
We use a dropout rate 0.25, layer normalization and an initial learning rate 0.0045. 
The learning rate followed a cosine annealing procedure~\citep{loshchilov2017sgdr}.

\subsection{ListOps}

The inputs are sequences of token IDs (integers) with maximum length 2048. 
We padded all sequences to this maximum length and used an embedding layer to encode them into 128 channels. 
We used 10 \ourblock s, each with filter size 4, 128 channels, and the resolution fading memory mechanism. 
We performed mean-pooling only on timesteps that were actual inputs to generate classification logits. 
We trained the network for 100 epochs with batch size 50. 
We used the AdamW optimizer with a weight decay rate 0.03. 
The learning rate was set to 0.003 after 1 epoch of linear warmup and then followed by a cosine annealing. 
We used a dropout rate 0.1 and batch normalization instead of layer normalization in this experiment. 

\subsection{PTB-XL}

The inputs are time series with 1000 timesteps and 12 channels. 
We transferred the architecture and learning rate from our Sequential CIFAR-10 experiments. 
We used layer normalization, a dropout rate 0.2, and the AdamW optimizer with weight decay rate 0.06. 
We trained the network for 5 epochs of linear warmup followed by 95 epochs of cosine learning rates.

\subsection{Autoregressive generative modeling}
\label{app:ar-details}

CIFAR-10 images were reshaped into sequences of length 1024, each with three channels. 
Let $x_i \in \mathbb{R}^3$ denote the RGB values (rescaled and centered between $[-1, 1]$) of the $i$-th pixel in the sequence. 
To model the conditional distributions $p(x_i|x_1, \cdots, x_{i-1})$ for $i = 1, 2, \cdots, 1024$, the inputs of the network were chosen as $(0, x_1, x_2, \cdots, x_{1023})$, with the target output sequence being the parameters for distributions of $x_1, x_2, \cdots, x_{1024}$. 
We used the same discretized mixture of Logistics distribution and linear correlation structure from \citet{salimans2017pixelcnn} to model the RGB values.
We used 48 \ourlayer s with filter size 4, 512 channels, and the resolution fading memory mechanism.
Additionally, we added a 1x1 convolutional layer close to the output end of the residual block. 
The causal structure of \ourlayer\ ensures that the $i$-th output only depends on the first $i$ input elements, which keeps the conditional dependency pattern required for autoregressive modeling.  
We trained the network for 250 epochs with batch size 64.
We used Adam optimizer with default hyperparameters. 
We used a dropout rate 0.1, layer normalization, and an initial learning rate 0.001. 
The learning rate followed a cosine annealing procedure. 

\section{Additional Results}
\label{app:other-exps}

\begin{figure}[h]
    \includegraphics[width=0.6\textwidth]{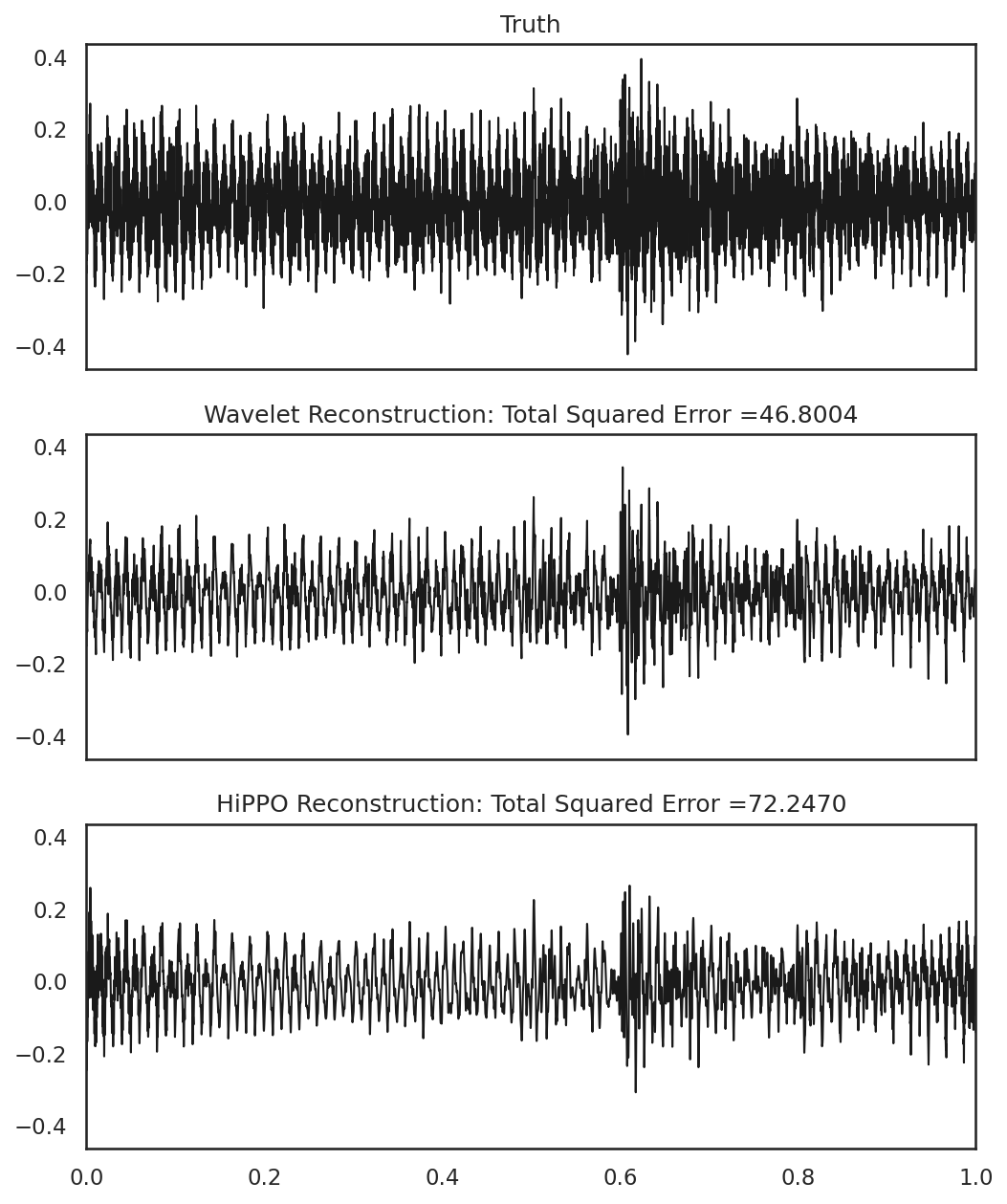}
    \centering
    \vspace{-3mm}
    \caption{Audio signal reconstruction from wavelet and polynomial bases. \emph{Top:} 1 second of a raw audio waveform with 16,384 time points. \emph{Middle:} Reconstruction from a 10-level Daubechies-4 wavelet decomposition with 2068 coefficients. \emph{Bottom:} Reconstruction from HiPPO~~\citep{gu2020hippo} via projection to orthogonal polynomials using 2068 coefficients (see \cref{sec:related-work}).}  %
\vspace{-4mm}
\label{fig:audio}
\end{figure}

\end{document}